\documentclass[letterpaper]{article} 
\usepackage[]{aaai24}  
\usepackage{times}  
\usepackage{helvet}  
\usepackage{courier}  
\usepackage[hyphens]{url}  
\usepackage{graphicx} 
\usepackage{natbib}  
\usepackage{caption} 
\usepackage{algorithm}
\usepackage[noend]{algorithmic}
\usepackage[table,xcdraw]{xcolor}
\usepackage{enumitem}

\usepackage{newfloat}
\usepackage{listings}
\DeclareCaptionStyle{ruled}{labelfont=normalfont,labelsep=colon,strut=off} 
\floatstyle{ruled}
\newfloat{listing}{tb}{lst}{}
\floatname{listing}{Listing}
\usepackage{cite}
\usepackage{caption}
\usepackage{subcaption}
\usepackage{amsmath,amssymb,amsfonts, mathrsfs}
\usepackage{algorithmic}
\usepackage{graphicx}
\usepackage{subfiles}
\usepackage{textcomp}
\usepackage{booktabs}
\usepackage{multirow}
\usepackage{amsthm}
\newtheorem{definition}{Definition}
\newtheorem{theorem}{Theorem}
\newtheorem{example}{Example}
\newtheorem{proposition}{Proposition}

\newtheorem{lemma}{Lemma}

\newcommand{\ensemble}[1]{\mathcal{E}_{\mathcal{#1}}}

\title{
Efficient Generation of Hidden Outliers for Improved Outlier Detection
}
\author {
    Jose Cribeiro-Ramallo\textsuperscript{\rm 1},
    Vadim Arzamasov\textsuperscript{\rm 1},
    Klemens Böhm\textsuperscript{\rm 1}
}
\affiliations {
    \textsuperscript{\rm 1}Karlsruhe Institute of Technology
    \{jose.cribeiro, vadim.arzamasov, klemens.boehm\}@kit.edu
}

\begin{document}

\maketitle

 \begin{abstract}
 Outlier generation is a popular technique used for solving important outlier detection tasks. 
 Generating outliers with realistic behavior is challenging. 
 Popular existing methods tend to disregard the ``multiple views'' property of outliers in high-dimensional spaces.
 The only existing method accounting for this property falls short in efficiency and effectiveness. 
 We propose \textsc{Bisect}, a new outlier generation method that creates realistic outliers mimicking said property. 
 To do so, \textsc{Bisect} employs a novel proposition introduced in this article stating how to efficiently generate said realistic outliers. 
 Our method has better guarantees and complexity than the current methodology for recreating ``multiple-views''. 
 We use the synthetic outliers generated by \textsc{Bisect} to effectively enhance outlier detection in diverse datasets, for multiple use cases. For instance, oversampling with \textsc{Bisect} reduced the error by up to 3 times when compared with the baselines.
 \end{abstract}

\section{Introduction}

\paragraph{Motivation.}
Outliers are observations in a dataset that stand out from the rest.
They represent rare or surprising events, making outlier detection important in many applications~\cite{fraud, outbook}. 
In a nutshell, outlier detection can be approached in two ways. 
The conventional approach is to use one outlier detection model fitted to all dimensions, the \emph{full-space approach}. 
But according to the ``multiple views'' property~\cite{multview}, points often are outliers only in some sets of dimensions (subspaces).
Together with the curse of dimensionality, this often renders the approach ineffective in high-dimensional data~\cite{aggarwal2001surprising, HiCS}. 
A very common alternative approach 
uses an ensemble of outlier detection models trained on several smaller subspaces~\cite{randomsubsp,Ensembook}, dubbed \emph{subspace approach}. 
We will build on this distinction in what follows.

Synthesizing outliers may improve outlier detection~\cite{optimalclass}. 
For instance, one may present synthetic outliers to a domain expert for labeling and later refine the outlier detection model using their feedback~\cite{hidden}.
Next, utilizing synthetic outliers allows to reframe 
outlier detection as a classification problem~\cite{GAAL}.
However, if objects generated are too far from the inliers, not representative of the outlier class, or even if there are too many artificial samples \cite{GAAL, Outsurvey}, the classification boundary degrades \cite{hempstalk2008one}.
This calls for a ``careful'' generation procedure, as opposed to a random one.

Most outlier generation approaches fall into one of two groups. 
\emph{Original space generators}~\cite{optimalclass, active} generate outliers using the original domain of the data.
\emph{Embedded space generators}~\cite{GAAL, AnoGAN} create artificial outliers using a representation of the embedded latent space of the data. 
When the dimensionality is high, original space generators in particular tend to create outliers arbitrarily distant from the inliers~\cite{Outsurvey, onthesurprising}. 
Embedded-space generators in turn produce samples in a lower-dimensional space to solve this problem. 
But coming up with an embedding model that accurately represents the domain is challenging because of the need for specific assumptions that are difficult to verify~\cite{dimred2}, 
or because a lot of data is required~\cite{dimred1}. 
So these models tend to be unsuitable if not enough is known about the data.
In either case, synthetic outliers may not demonstrate the common ``multiple view'' characteristic of high-dimensional outliers~\cite{multview}, thus failing to replicate the outlier class. 

\emph{Hidden outliers}~\cite{hidden} are outliers that represent the disagreement between full-space and subspace outlier detection approaches: 
By definition, hidden outliers are detectable by either a full-space approach 
or a subspace approach, but not both. 
Hence, they feature the ``multiple views'' property and at the same time tend to be close to the inliers. 
This is because distant points tend to be detectable by both approaches and thus are not hidden. 

\begin{figure}[t]
\centering
    \begin{subfigure}[b]{.49\linewidth}
    \centering
        \includegraphics[width = .7\textwidth]{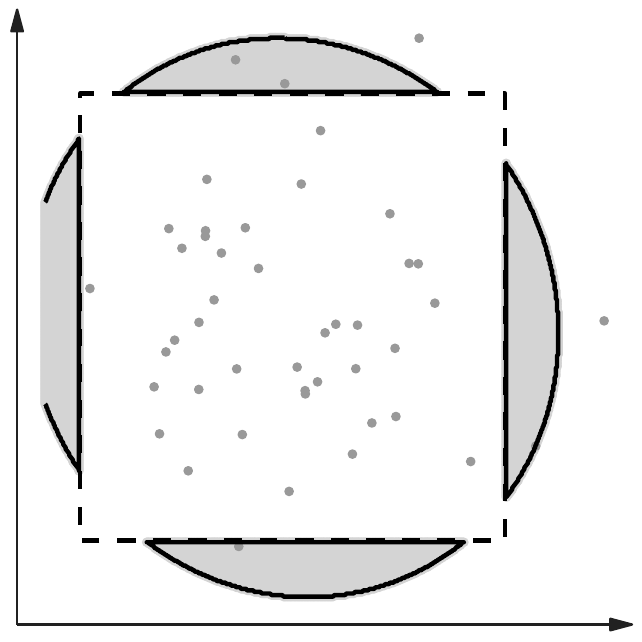}
        
    \end{subfigure}
    \hfill
    \begin{subfigure}[b]{.49\linewidth}
    \centering
        \includegraphics[width = .7\textwidth]{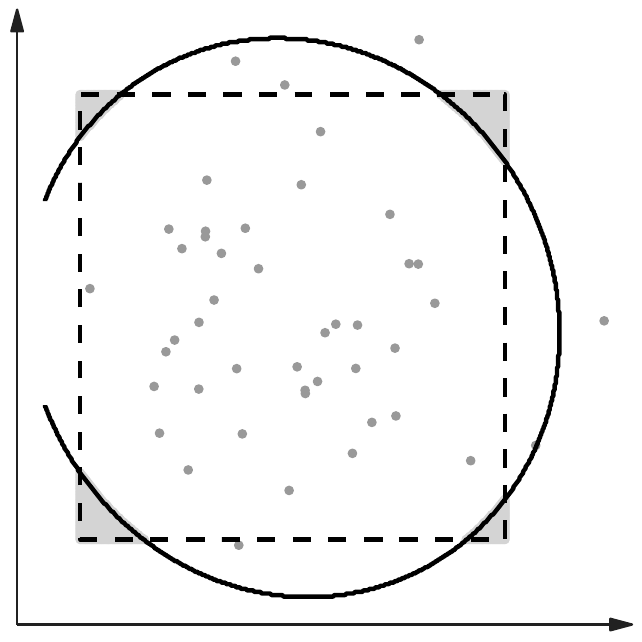}
        
    \end{subfigure}
    \caption{Example of regions of hidden outliers. }
    \label{fig:HORegion}
\end{figure}

\begin{example}
In Figure~\ref{fig:HORegion}, the shaded areas represent regions where hidden outliers may occur. 
The solid ellipse line delimits the boundary of the two-dimensional full-space outlier detector, while the dashed lines constitute the boundary of the subspace detector, composed of detectors fitted in each one-dimensional subspace. 
Points in the shaded areas on the left can only be detected by the subspace detector, and the ones on the right by the full-space detector.
\end{example}

\paragraph{Challenges.}
Despite their desirable properties, generating hidden outliers is not trivial. 
First, there are no known guarantees that hidden outliers can be generated. 
Next, even if they can, it is generally impossible to describe the regions of the space from which they originate analytically and sample from there. 
It is necessary to generate candidate points and verify which ones qualify as hidden outliers.
However, generating too many poor candidates renders computations intractable.
Namely, the definition of hidden outliers builds on subspace outliers, while subspace outlier detection uses up to $2^N-1$ models, one for each subspace, where $N$ is the number of dimensions. 
Another issue is that an ill-designed candidate generation process could leave parts of the space unexplored. 
Finally, quantitative evaluation criteria for hidden outliers currently do not exist. 
All this calls for an efficient generation method and a procedure to assess its utility.
The only existing generator for hidden outliers, \textsc{Hidden}~\cite{hidden}, heavily relies on its hyperparameter which one cannot readily tune, as we will explain. It is unclear whether a value of this hyperparameter exists that ensures both efficient synthesis and high quality of hidden outliers. 
 
\paragraph{Contributions}
Our contributions are as follows.
(1)~We prove the existence of hidden outliers. In particular, 
we formulate and prove the so-called 
hidden outlier existence proposition which asserts that it is always possible to generate a hidden outlier between an inlier and an outlier of the full-space approach. 
(2)~Leveraging this proposition, we propose a new method to efficiently generate hidden outliers, which we call \textsc{Bisect}. 
Importantly, \textsc{Bisect} 
gives certain guarantees in generating hidden outliers
and does not require any external hyperparameters.
(3) We propose a methodology to improve outlier detection using hidden outliers. 
We apply it to various datasets to showcase the value of hidden outliers synthesized by \textsc{Bisect}. 
There are certain obstacles that are in the way of very general claims regarding any potential superiority of BISECT for outlier detection, as we will explain in Sections 4 and 5. 
Having said this, we observed significantly reduced error rates in one-class classification and supervised outlier detection in our experiments compared to established methods. 
(4)~We share the code of \textsc{Bisect} and of our experiments.\footnote{\url{https://github.com/jcribeiro98/Bisect}}

\section{Related Work}

\paragraph{Outlier Generation.}
There exists an extensive body of work on synthetic outlier generation in tabular data~\cite{Outsurvey}. Some of these methods do~\cite{Oversamplingwithoutliers}, while others do not~\cite{GAAL}, require genuine outliers. One approach is to sample points from a distribution distinct from that of inliers. Other representative approaches include adding noise to the inliers or permuting their attribute values. For a comprehensive review and taxonomy of outlier generation methods, we refer to~\cite{Outsurvey}. However, the vast majority of these methods do not verify the ``multiple-view'' property, and may produce outliers not accurately representing the outlier class.

\paragraph{Generation of Hidden Outiers.}
To our knowledge, there is only one method, \textsc{Hidden}~\cite{hidden}, that specifically replicates the ``multiple views'' property by generating hidden outliers.
At each iteration, \textsc{Hidden} randomly samples a candidate point from a hypercube centered at a random genuine point and checks whether this candidate is a hidden outlier. 
The lateral size of each hypercube is $\varepsilon \cdot \textit{range}$, where \textit{range} represents the maximum range of values among the genuine data points in any dimension. 
The hyperparameter $\varepsilon \in \left[0,1\right]$ controls the size of each hypercube. However, choosing an appropriate $\varepsilon$ value can be challenging. 
If values of $\varepsilon$ are too small, the generation of hidden outliers may fail. 
Values close to 1 in turn often lead to inefficiencies due to the generation of many poor candidates. 
There is no recommended universal value for $\varepsilon$: An $\varepsilon$ value that is efficient for certain datasets might not work well for different ones.
Additionally, different $\varepsilon$ values result in distinct and uncontrollable probability distributions of generated hidden outliers, potentially leaving out important regions in the data space. 
See~Appendix for demonstrations. Our method, \textsc{Bisect}, does not rely on such a hyperparameter and exhibits more predictable behavior. 
We will compare \textsc{Bisect} with \textsc{Hidden}.

\textbf{Evaluation of Artificial Outliers.}
We are not aware of any established methodology to evaluate hidden outliers. 
The paper~\cite{hidden} primarily focuses on using hidden outliers for exploratory purposes. 
To demonstrate the utility of \textsc{Bisect} and its advantages over \textsc{Hidden}, we leverage evaluation techniques from outlier generation in general.
Synthetic outliers have several applications within outlier detection~\cite{Outsurvey}. 
Common use cases in the literature are self-supervised and supervised outlier detection~\cite{GAAL,Dlamini2021}.

Self-supervised outlier detection uses synthetic outliers as the positive class to construct a two-class classifier. Subsequently, this classifier is applicable in various tasks~\cite{Chandolasurvey}, particularly in one-class classification. 
The effectiveness of the generated outliers determines the performance of the resulting classifier.

The supervised case for outlier detection is akin to an unbalanced binary classification problem with an extremely scarce minority class \cite{outbook}. 
To address imbalanced data, a common technique is oversampling the minority class \cite{SMOTE, Oversamplingwithoutliers,cWGAN}. 
This has also been applied in the supervised outlier detection case~\cite{Dlamini2021,extreme_imbalance}.
The quality of the generated data depends on the performance gain achieved compared to the unbalanced classifier. 

Oversampling and self-supervised learning are among the most frequently used applications of synthetic outliers in the literature~\cite{occsurvey, Chandolasurvey}. Furthermore, theoretical studies have demonstrated that self-supervised classifiers are particularly well-suited for one-class classification \cite{optimalclass}. Hence, we will evaluate synthetic hidden outliers with these tasks.

\section{Our Method: \textsc{Bisect}}
This section presents \textsc{Bisect}, a novel method for generating hidden outliers. 
We start with the necessary definitions and derive new theoretical results regarding the existence of regions with hidden outliers. 
Subsequently, we explore how these findings can be used to generate hidden outliers effectively. 
Finally, we comprehensively describe \textsc{Bisect} and analyze its properties.
All propositions and theorems have their corresponding proofs in the appendix. 

\subsection{Hidden Outliers and the Hidden Region}

Let $X = \mathbb{R}^d$ be a metric space, and $D=\{x^i\}\subset X$ be a finite dataset.
Let $\mathcal{M(\cdot)}:X\longrightarrow~\{0,1\}$ denote an outlier detector model, and let $\mathcal{M}=\mathcal{M}(D)$ be its fitted version with $D$. 
Next, let the closed set $R(\mathcal{M})=\left\{x\in X| \mathcal{M}(x) = 0 \right\}$ be the acceptance region (i.e., the one containing inliers) of $\mathcal{M}$ with the boundary $\partial R(\mathcal{M})$. Further, let $\mathcal{M}_S:S\subset X\longrightarrow\{0,1\}$ be a detector fitted with the projected dataset over the subspace $S\subset X$, 
$D|_{S} = \left\{x^i|_S\right\}$. 
For $\mathcal{M}(\cdot)$ and the set of subspaces $\{S_i\}$, 
we define the subspace ensemble of $\mathcal{M}(\cdot)$ as the mapping  $ \mathcal{E}_\mathcal{M}: X \longrightarrow \{0,1\}$, such as $\ensemble{M}(x) = a\left(\mathcal{M}_{S_1}(x|_{S_1}),\dots,\mathcal{M}_{S_m}(x|_{S_m})\right)$, where $a$ is an aggregation function.
We now formally define hidden outliers and hidden region:

\begin{definition} (Hidden region, hidden outlier, adversary)
    Given a fitted outlier detector $\mathcal{M}$, 
    a subspace ensemble $\mathcal{E}_\mathcal{M}$, 
    and aggregation function $a =\textrm{max}$, define the sets:
    \begin{align*}
        H_1(\mathcal{M}) &= R(\mathcal{M})\setminus R(\ensemble{M}),\\
        H_2(\mathcal{M}) &= R(\ensemble{M}) \setminus R(\mathcal{M}).
    \end{align*}
    The union $H_1(\mathcal{M})\cup H_2(\mathcal{M})$ is the hidden region of $\mathcal{M}$. A hidden outlier is a point $h \in X$ in the hidden region of $\mathcal{M}$.
    The outlier detector $\mathcal{M}(\cdot)$ used is the adversary of $h$.  
\end{definition}
For instance, in Figure~\ref{fig:HORegion}, the shaded area on the left plot is $H_1$, and $H_2$ is the shaded area on the right.

By this definition, hidden outliers are points that are outliers in certain subspaces of $X$ but inliers in the original space, or vice-versa. This implies that hidden outliers necessarily exhibit the ``multiple views'' property when $\{S_i\}$ are generated with the canonical basis of $X$. 

Without a guarantee that a hidden region exists, 
effectively searching for good candidates is hard. 
The following proposition addresses this concern and provides guidance on where one could generate hidden outliers. 
The conditions presented are sufficiently versatile to occur in real scenarios. 
\begin{proposition}\label{convexexistence}("Hidden outlier existence"):
    Let $x$ and $y$ be points in the previously defined metric space such that $x \in R(\mathcal{M})$ and $y \notin R(\mathcal{M})$. Assume that there exists a point $z$ in the convex combination of $x$ and $y$ such as $z \in \partial R(\mathcal{M}) \Rightarrow z \notin \partial R(\ensemble{M})$. 
    Then there exists 
    $z'$ in the convex combination of $x$ and $y$ such as $z' \in H_1(\mathcal{M}) \cup H_2(\mathcal{M})$.
\end{proposition}

\begin{figure}
\begin{subfigure}{.45\linewidth}
  \centering
  \includegraphics[width=.8\linewidth]{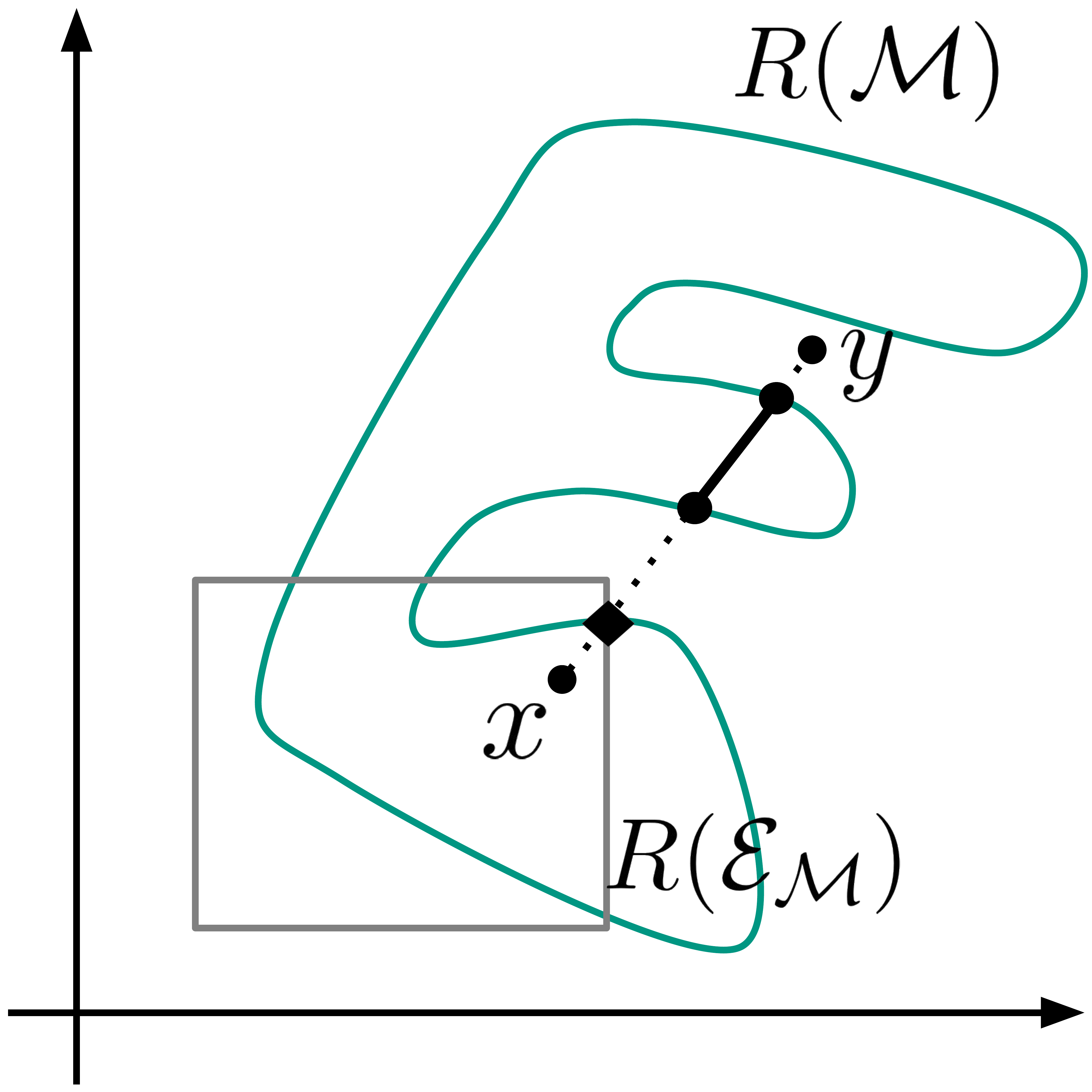}
  \caption{}
  \label{sfig:ho1} 
\end{subfigure}
\begin{subfigure}{.45\linewidth}
  \centering
  \includegraphics[width=.8\linewidth]{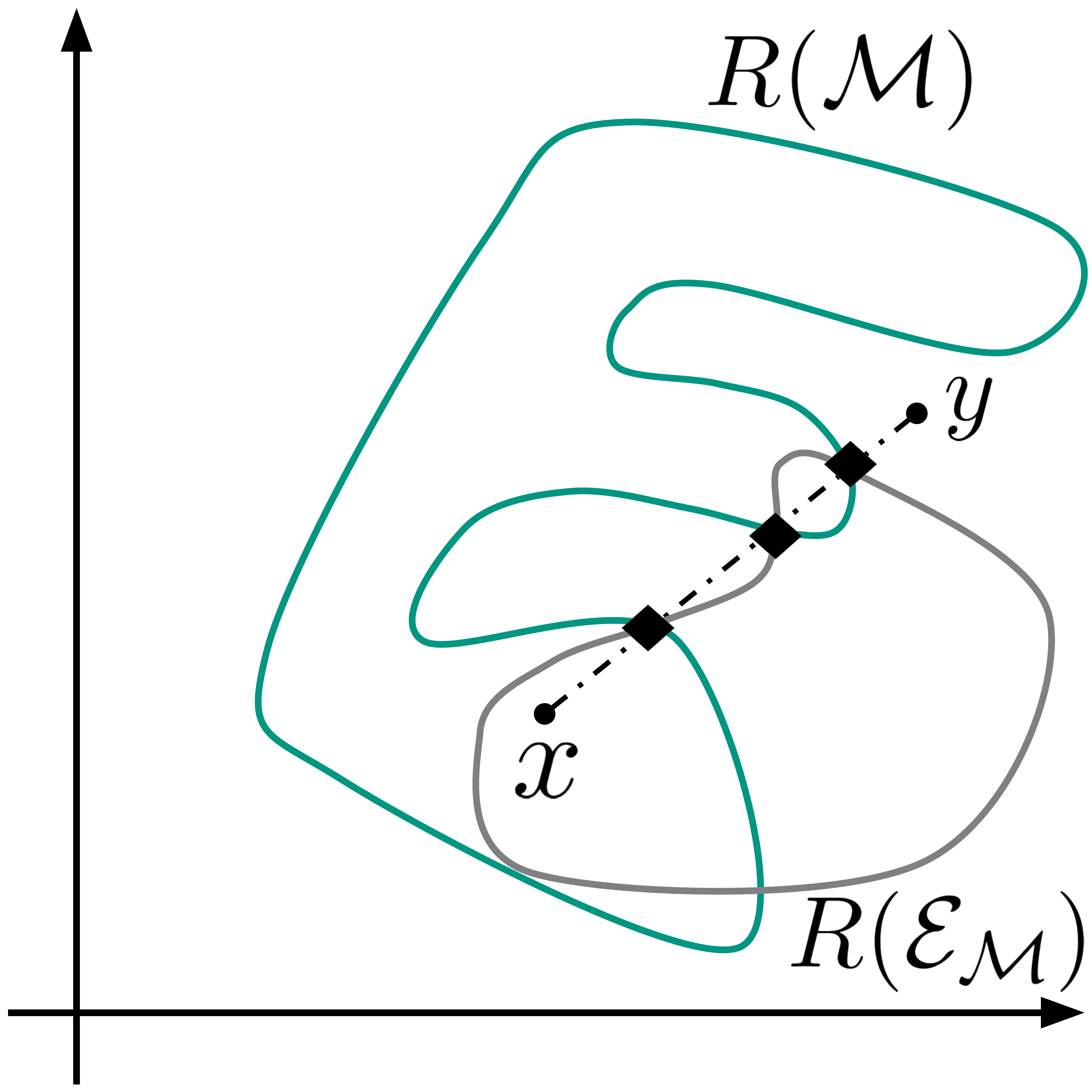}
  \caption{}
  \label{sfig:ho2} 
\end{subfigure}
\caption{Examples of different hidden regions. $R(\mathcal{M})$ is marked in green, $R(\ensemble{M})$ in grey, and the convex combination of $x$ and $y$ is represented with a dashed line.} 
\label{fig:ho}
\end{figure}
 The existence of $z$ condition in the proposition guarantees that the convex combination crosses the hidden region.
 \begin{example}
 In Figure \ref{fig:ho} the dashed line marks the convex combination of $x$ and $y$. 
 We marked with a straight line the intersection of the convex combination with the hidden region. 
 The condition in Proposition \ref{convexexistence} rules out cases like the one in Figure \ref{sfig:ho2}, where there is no intersection with the hidden region. 
 Cases like in Figure \ref{sfig:ho1} that cross the intersection of boundaries but contain hidden outliers are accepted.   
 \end{example} 

We can limit the search for candidate hidden outliers to the convex combination of two points. This shifts the problem from searching in $X$ to searching in the image of $\alpha(t) = ty + (1-t)x$, i.e., in the set of points 
$ty + (1-t)x$ for $t\in [0,1]$. 
The following section elaborates on this. 

\subsection{Generating Hidden Outliers by Finding Roots}

We reformulate $\alpha(t)$ so that the points of the image that are hidden outliers are now its zeroes. 
This will allow us to use the bisection method to generate hidden outliers. 
Let $F$ be an indication function, such as: \[
    F(x) = 
    \begin{cases}
        1, & \text{if } \mathcal{M}(x) = \ensemble{M}(x) = 1, \\
        0, & \text{if } \mathcal{M}(x) \neq \ensemble{M}(x), \\
        -1, & \text{if } \mathcal{M}(x) = \ensemble{M}(x) = 0.
    \end{cases}
\]
The roots of the function $f = (F \circ \alpha) (t)$ are $t \in [0,1]$ such that $\alpha(t)$ are hidden outliers. To ensure convergence of the bisection method to a root of $f$, $f$ must be monotonic. The following theorem states that $f$ is monotonic when the image of $\alpha$ crosses the hidden region once. 

\begin{theorem}\label{convergence}("Convergence into a hidden outlier")
Let $f$ be defined as before. 
Assume that at most exist, and are unique, $z \text{ and } z'$ in the convex combination of $x\in R(\mathcal{M})$ and $y \notin R(\mathcal{M})$ such as: $z \in \partial R(\mathcal{M})$, $z' \in \partial R(\ensemble{M})$, $z \neq z'$, and both verify the last condition of proposition \ref{convexexistence}. 
Then the bisection method will converge to a root of $f$.
\end{theorem}

The theorem straightforwardly extends to the case when the image of $\alpha$ crosses the hidden region a finite number of times. 
This is because segmenting the interval $[0,1]$ sufficiently many times yields subintervals that fulfill the conditions in Theorem \ref{convergence}.

\begin{example}
    Figure \ref{fig:INTERVAL} illustrates the process, which we will term the ``cut trick.'' In the right plot, six equally spaced points denoted as $t_1, \ldots, t_6 \in [0, 1]$ divide the original line segment (on the left plot) into five equal-length parts. Dots represent inliers (in  $R(\mathcal{M})$), while crosses represent outliers (not in $R(\mathcal{M})$). Although the original image for $t \in [0, 1]$ does not satisfy the conditions of the theorem, the segments corresponding to $[t_2, t_3]$, $[t_4, t_5]$, and $[t_5, t_6]$ do meet these conditions. Consequently, the bisection method can generate hidden outliers within each of these segments.
\end{example}
\begin{figure}
    \begin{subfigure}{.45\linewidth}
        \centering
        \includegraphics[width=.8\linewidth]{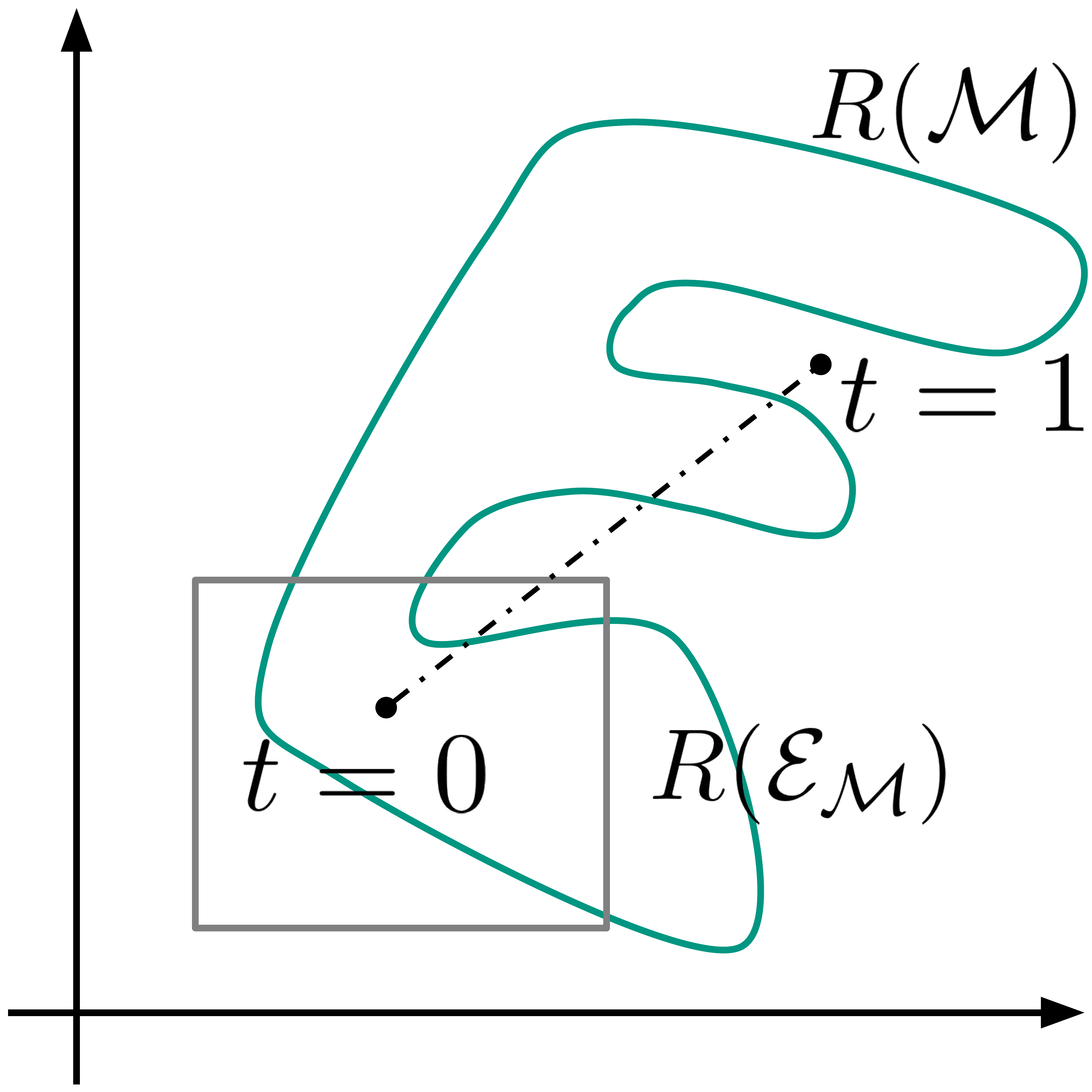}
        \caption{Before cuts}
        \label{sfig:INTERVAL1}
    \end{subfigure}
    \begin{subfigure}{.45\linewidth}
        \centering
        \includegraphics[width=.8\linewidth]{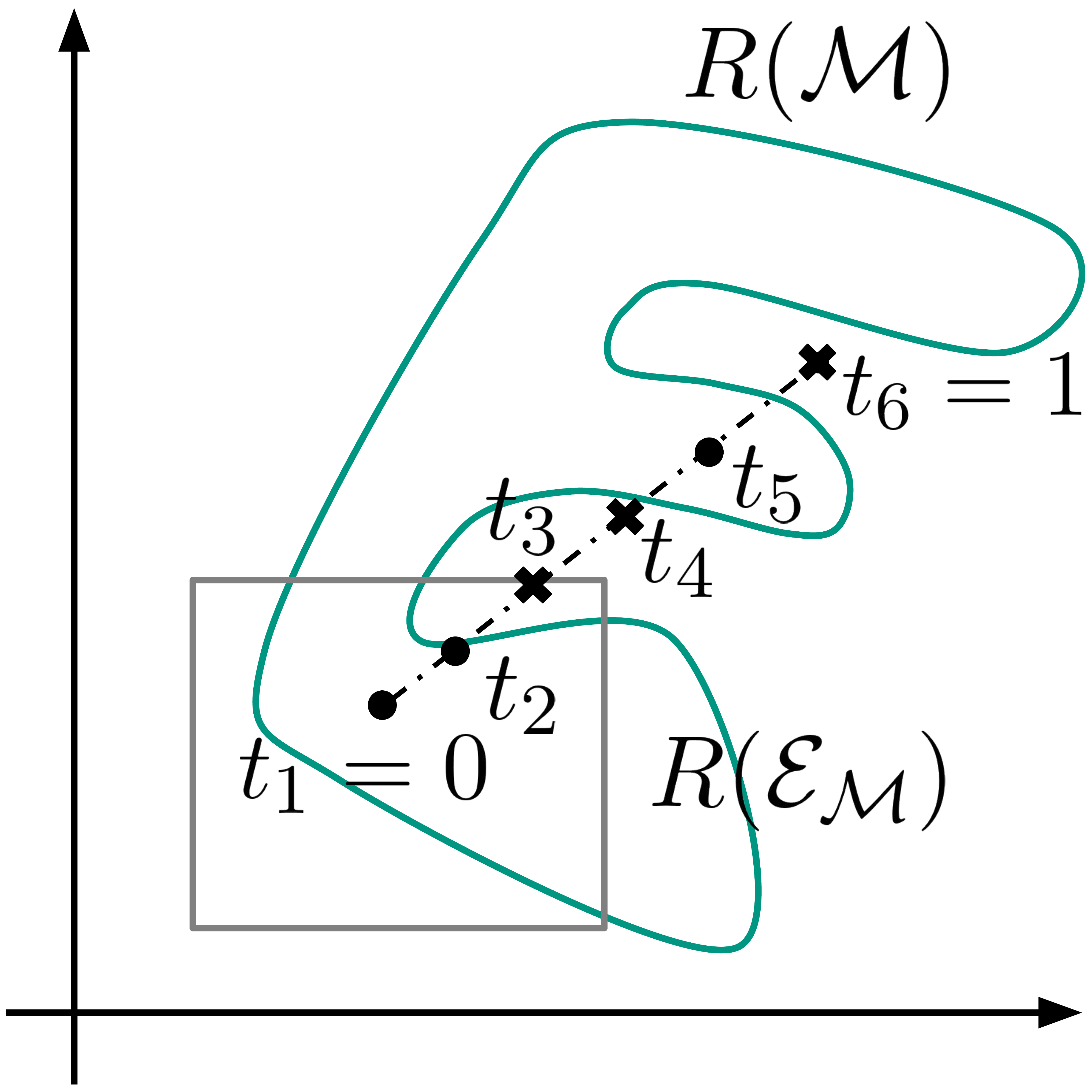}
        \caption{After cuts}
        \label{sfig:INTERVAL2}
    \end{subfigure}
    \caption{An example of the ``cut trick'' 
    .} 
    \label{fig:INTERVAL}
\end{figure}

\subsection{The \textsc{Bisect} Algorithm}
Combining all the above results, we propose \textsc{Bisect} (Algorithm \ref{Bisect}). It works with the following steps: 
\begin{enumerate}
    \item Randomly select an origin, $\Tilde{o} \in R(\mathcal{M})\cap D$. 
    If $\mathcal{M}$ is a composition of two functions, $m:X\rightarrow \mathbb{R}$, a scoring function, and $\mathcal{T}:\mathbb{R}\rightarrow \{0,1\}$, a threshold function, \textsc{Bisect} selects $\Tilde{o}$ using the weighted distribution obtained by the scoring of $m$. We do this as we found it to be superior to random selection during preliminary experiments. \label{origin} 
    \item Select a random direction $v$ in $\mathbb{S}^{d-1}(\mathbb{R})$, the $d$-dimensional unitary sphere. Define a sufficiently large $L\in\mathbb{R}$ 
    so that the point $y = \Tilde{o} + Lv$ is not in $R(\mathcal{M})$ (Lines~1--3). If $y\in R(\mathcal{M})$, restart \textsc{Bisect} with a new origin (Lines~5--7).
    \label{L} 
    \item Obtain a subinterval of the convex combination $\alpha(t) = ty + (1-t)\Tilde{o}$ using the ``cut trick'' (Line~4). 
    \item 
    For this subinterval, perform the bisection method for $f=F\circ\alpha$ to obtain one hidden outlier (Lines 9--14).  
    \label{Calculations}
\end{enumerate}

\begin{algorithm}[t]
    \caption{\textsc{Bisect}}
        \label{Bisect}
    \begin{algorithmic}[1]
    \REQUIRE $\Tilde{o},~n_\text{cuts},$ and $\mathcal{M}~\text{\&}~\ensemble{M}$ fitted with $D$.
    \STATE $v \gets Unif(1,\mathbb{S}^{d-1}(\mathbb{R})) $ 
    \STATE $l \gets \underset{x\in D}{\max} \|x\|$
    \STATE $L = l + \textrm{Unif}\left(1, (-\frac{l}{2},l)\right)$
    \STATE $\{a,b\} \gets \textsc{Interval}(\Tilde{o},v,0,L,n_\text{cuts})$
    \IF{$\{a,b\} = \emptyset$}
        \STATE Select new origin $\Tilde{o}'$
        \STATE \textsc{Bisect}($\Tilde{o}',n_\text{cuts},\mathcal{M},\ensemble{M}$)
    \ENDIF
    \STATE $c \gets \frac{a + b}{2}$
    \WHILE{$F(c) \neq 0$}
    \IF{$F(a) = F(c)$}
    \STATE $a \gets c$
    \ELSE 
    \STATE $b \gets c$
    \ENDIF
    \STATE $c \gets \frac{a + b}{2}$
    \ENDWHILE
    \RETURN $c$
    \end{algorithmic}
\end{algorithm}

We now discuss the complexity of \textsc{Bisect}.


\begin{proposition}\label{complexity}
    Let $n_\text{subs}$ be the number of selected subspaces for the ensemble. 
    The worst-case complexity of \textsc{Bisect} is  $\mathcal{O}\left(\log(\frac{L}{n_\text{cuts}}) \cdot (n_\text{subs}+n_\text{cuts}) \cdot \aleph\right)$ where $\aleph$ is the inference complexity of the adversary $\mathcal{M}(\cdot)$.
\end{proposition}
According to this proposition, the parameter $n_\text{cuts}$ used for the cut trick affects the complexity of \textsc{Bisect}. In practice, low values of $n_\text{cuts}$ are sufficient, as
we now explain.

\begin{algorithm}[t]
    \caption{\textsc{Interval} (cut trick)}
    \label{Interval}
    \begin{algorithmic}[1]
        \REQUIRE $\Tilde{o},v,a,b,n_\text{cuts}$
        \STATE Initialize $I = \emptyset$
        \STATE $\{t_i\}_{i=1}^{n_\text{cuts} + 1} \gets \text{sequence}(a,b,n_\text{cuts})$
        \FOR{ $i$ in  $1$ to $n_\text{cuts}$}
            \STATE check$_i$ $\gets \mathcal{M}(\Tilde{o} + t_iv)$
            \IF{check$_i$$\neq$check$_{i-1}$}
            \STATE Add $\{t_{i-1},  t_i\}$ to $I$
            \ENDIF
        \ENDFOR
        \STATE $\{l,r\} \gets Unif(1,I)$
        \RETURN $l,r$
    \end{algorithmic}
\end{algorithm}

\subsubsection{Selecting the Number of Cuts.} 

The number of iterations needed for the bisection method to generate a hidden outlier in the worst case can be calculated as~$n_\text{iter} = \textit{Int}(\frac{\log({L}/{n_\text{cuts}}) - \log(Err)}{\log(2)}  - 1)$. Where $Err$ is the error rate and $Int(\cdot)$ is the mapping to the nearest integer. For $Err=0.05$ and a unitary interval, five cuts on the interval are the smallest number of cuts that results in only a single iteration needed to converge. 
Since we have found through experimentation that five cuts are usually enough to satisfy the conditions of Proposition \ref{convergence}, we recommend setting $n_\text{cuts}=5$.

\subsubsection{Complexity Comparison to \textsc{Hidden}.} 
For a fixed $n_\text{cuts}$, $\mathcal{O}\left(n_\text{samp} \cdot \log(L) \cdot n_\text{subs} \cdot \aleph\right)$ is the worst case complexity of \textsc{Bisect} to generate $n_\text{samp}$ hidden outliers. The corresponding worst-case complexity of \textsc{Hidden} is $\mathcal{O}\left(n_\text{samp} \cdot\mathcal{P}(D, \varepsilon)\cdot n_\text{subs} \cdot \aleph\right)$, where the parameter $\mathcal{P}(D, \varepsilon)$ inversely correlates with the probability of generating a hidden outlier for a given data at a particular $\varepsilon$~\cite{hidden}. It is impossible to estimate this parameter beforehand. We expect that in practice it will be greater and more irregular than $\log(L)$. 
\section{Experiments}

In this section, we 
empirically demonstrate the utility of hidden outliers
generated with \textsc{Bisect}.
First, we describe the synthetic and real datasets and configurations of hidden outlier generation methods for our experiments. 
Next, we compare the runtime performance of \textsc{Bisect} and \textsc{Hidden} on datasets with varying complexity.
After that, we study how various outlier detection tasks can benefit from generating hidden outliers. 
To finalize, we briefly comment on the limitations of the experimental study at the end of the section.
We implemented the experiments in R and Python\footnote{Using \texttt{reticulate} \cite{reticulate}.} and run them on a ThinkPad P14s-gen2 with 16GB of RAM using Ubuntu v22.04.1. 
For all calculations, used the CPU, a Ryzen PRO 7 5850u.

\subsection{Datasets}

\subsubsection{Synthetic Data.}

\begin{table}[t]
\centering
\begin{tabular}{ccc}
\toprule
\# Clusters & \# Features & \# Observations \\ 
\midrule
1, 2, 5 & $7, 15, 30, 50, 100, 150$ & $1000$ \\
\bottomrule
\end{tabular}
\caption{Characteristics of synthetic datasets.}
\label{sythetic}
\end{table}

To comprehensively evaluate hidden outlier generation methods under controlled conditions, we produce several synthetic datasets using a multidimensional clustered Gaussian distribution.
We systematically varied the number of clusters, features (columns), and observations (rows).
Table~\ref{sythetic} provides an overview of the dataset characteristics; for each combination of parameters, we randomly generated five synthetic datasets.

\subsubsection{Real Data.}

\begin{table}[t]
    \centering
    \begin{tabular}{llll}
    \toprule
        Dataset & \# Features & $\lvert D^\textit{full}\rvert$ & \# Outliers \\ 
    \midrule
        Wilt & 5 & 4819 & 261 \\ 
        Pima & 8 & 768 & 268 \\ 
        Stamps & 9 & 340 & 31 \\ 
        PageBlocks & 10 & 5473 & 560 \\ 
        Heart Disease & 13 & 270 & 120 \\ 
        Annthyroid & 21 & 7129 & 534 \\ 
        Cardiotocography & 21 & 2114 & 471 \\ 
        Parkinson & 22 & 195 & 147 \\ 
        Ionosphere & 32 & 351 & 126 \\ 
        WPBC & 33 & 198 & 47 \\ 
        SpamBase & 57 & 4207 & 1813 \\ 
        Arrhythmia & 259 & 450 & 206 \\ 
    \bottomrule
    \end{tabular}
    \caption{Characteristics of real datasets.}
\label{datasets}
\end{table}
We use real datasets\footnote{\url{https://www.dbs.ifi.lmu.de/research/outlier-evaluation/}} collected by \citeauthor{Campos2016} for evaluating unsupervised outlier detection. 
To ensure a reasonable runtime, 
we only consider datasets containing fewer than 10,000 observations and fewer than 1,000 features. 
To ensure reliable performance estimates, we only retain datasets with more than 30 outliers. 
Table~\ref{datasets} summarizes the resulting 12 datasets, referred to as $D^\textit{full}$. 
We also use modified versions of these datasets, denoted as $D$, where the outlier class is downsampled to 2\% of the total number of observations. 
Whenever possible, we obtain $D$ from~\cite{Campos2016}. 
For the Ionosphere and WPBC datasets, we created the downsampled versions ourselves.

\subsection{Method Configurations}
\label{ssection:configurations}
We use \textsc{Bisect} and \textsc{Hidden} with LOF~\cite{lof}
as an adversary. In some experiments, we also tried the KNN outlier detection method~\cite{kNN} 
as an alternative adversary for \textsc{Bisect} (\textsc{Bisect}\textsubscript{K}). 
We configure \textsc{Bisect} to use the weighted origin selection and fix the number of cuts to 5. For \textsc{Hidden}, we set $\varepsilon=0.1$, as we found larger values to be inefficient when generating hidden outliers for certain datasets, see~Appendix. 

To maintain tractability, we limited the number of subspaces used to generate hidden outliers to a maximum of 2048; 
when necessary, we used feature bagging without repetition~\cite{fb} for subspace selection.
When a generational method failed to produce a successful candidate within a 30-minute timeframe, we regarded the respective experiment as unsuccessful (marked as $ot$).
For all other methods used in our experiments, we adopted the default parameter settings suggested by the authors or provided by their respective implementations.

\subsection{Hidden Outlier Generation Efficiency}
\begin{figure*}
\centering
\begin{subfigure}{.245\textwidth}
  \centering
  \includegraphics[width=1\linewidth]{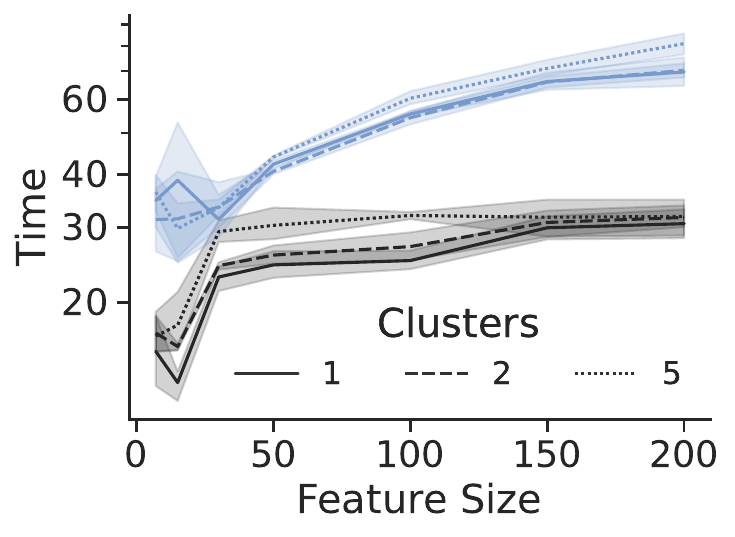}
  \caption{Synthetic datasets} 
  \label{sfig:synth} 
\end{subfigure}
\begin{subfigure}{.245\textwidth}
  \centering
  \includegraphics[width=1\linewidth]{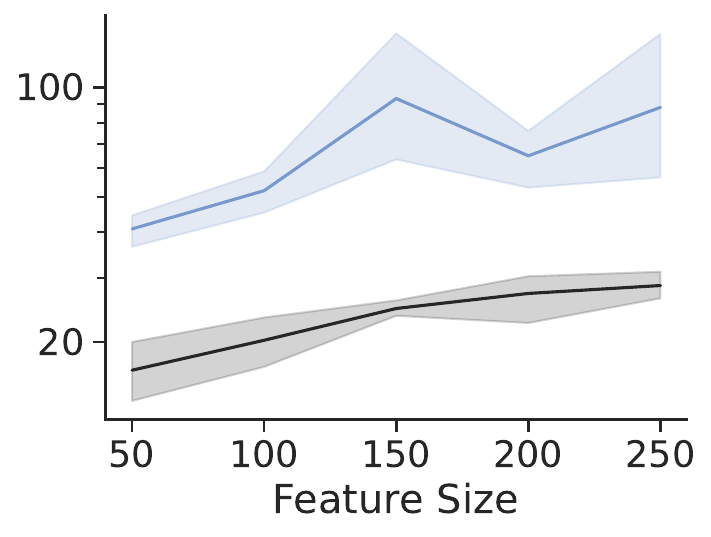}
  \caption{{Arrythmia}}
  \label{sfig:arr}
\end{subfigure}%
\begin{subfigure}{.245\textwidth}
  \centering
  \includegraphics[width=1\linewidth]{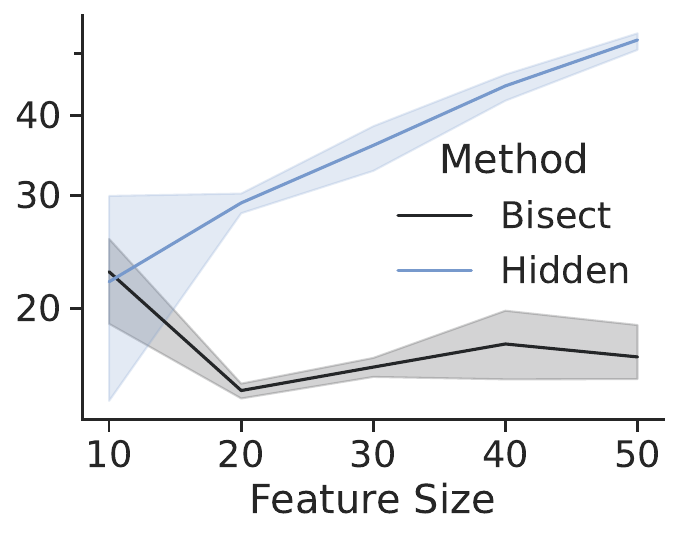}
  \caption{{SpamBase}}
  \label{sfig:sb}
\end{subfigure}%
\begin{subfigure}{.245\textwidth}
  \centering
  \includegraphics[width=1\linewidth]{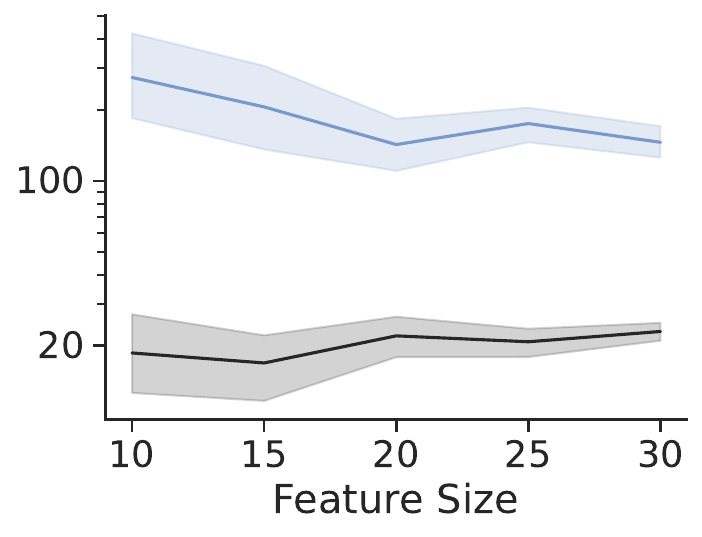}
  \caption{{WPBC}}
  \label{sfig:wpbc}
\end{subfigure}%

\caption{Time to generate 500 hidden outliers (in seconds) in synthetic and real data contingent on feature count.}
\label{fig:time-all}
\end{figure*}
In this section, we compare the time required by \textsc{Bisect} and \textsc{Hidden} to generate 500 hidden outliers (in seconds).

\subsubsection{Synthetic Data.}

Figure \ref{sfig:synth} shows the time needed to generate 500 hidden outliers with \textsc{Bisect} and \textsc{Hidden}, contingent on the number of features in a dataset.
\textsc{Hidden} requires 2--3 times more runtime than \textsc{Bisect}, and this difference increases with dimensionality. 
Dimensionality has a limited impact on the generation time when the number of subspaces is fixed, as expected by Proposition \ref{complexity}.
Additionally, the number of clusters in the data has little effect on the performance of \textsc{Hidden} or \textsc{Bisect}.

\subsubsection{Real Data.}
We used four real datasets with the highest feature count.
To achieve the desired dimensionality, we created five lower-dimensional datasets by selecting five distinct subsets of features with minimal overlap. 
Figure~\ref{fig:time-all} plots the time required to generate 500 hidden outliers using \textsc{Bisect} and \textsc{Hidden} for each dataset, depending on dimensionality. 
As with synthetic data, we observe that \textsc{Bisect} is several times faster than \textsc{Hidden}.

Table~\ref{tab:time} summarizes the generation times from our experiments, both with synthetic and real data, for \textsc{Bisect} and \textsc{Hidden}. 
It is evident that the performance of \textsc{Bisect} is significantly more predictable than that of \textsc{Hidden}~-- the interquartile range for the generation time is an order of magnitude lower for \textsc{Bisect}.

\begin{table}
\centering
\begin{tabular}{llllll|l}
\toprule
 & Min & $Q_1$ & $Q_2$ & $Q_3$ & Max & IQR \\ \midrule
{Bisect} & 9.9 & 15.3 & 19.6 & 25.2 & 44 & 9.9 \\
{Hidden} & 9.2 & 46.2 & 72.6 & 124.7 & 566 & 78.4 \\ \bottomrule
\end{tabular}
\caption{Time to generate 500 hidden outliers (in seconds).}
\label{tab:time}
\end{table}

\subsection{Outlier Detection Experiments}

\subsubsection{One-class Classification.}

\begin{table*}[ht]
\centering
\begin{tabular}{lllll|ll|llll}
\hline
Dataset & \textsc{Bisect} & \textsc{Hidden} & HB & LOF & $\textsc{Bisect}_\text{K}$ & KNN & IForest & DSVDD & OCSVM & GAAL \\ \hline
Wilt & \textit{\textbf{0.975{\color{red}*}}} & {0.942{\color{red}*}} & 0.520$^{\color{olive}\dagger}$ & 0.556 & \textbf{0.970{\color{red}*}} & 0.537 & 0,720 & 0.529 & 0.643 & 0.653 \\
Pima & 0.579$^{\color{olive}\dagger}$ & 0.499$^{\color{olive}\dagger}$ & 0.523$^{\color{olive}\dagger}$ & \textbf{0.695} & 0.563$^{\color{olive}\dagger}$ & \textit{\textbf{0.747}} & 0,718 & 0.511 & 0.664 & 0.498 \\
Stamps & \textbf{0.954{\color{red}*}} & 0.699$^{\color{olive}\dagger}$ & 0.952{\color{red}*} & 0.895 & \textit{\textbf{0.961}} & 0.957 & 0,935 & 0.766 & 0.877 & 0.913 \\
PageBlocks & \textbf{0.952} & 0.911& 0.895$^{\color{olive}\dagger}$ & 0.931 & \textbf{0.925{\color{red}*}} & 0.661 & \textit{0,957} & 0.676 & 0.619 & 0.415 \\
Heart Disease & 0.806 & 0.825 & \textit{\textbf{0.829}} & 0.839 & 0.839 & \textbf{0.813} & 0,799 & 0.737 & 0.801 & {0.818} \\
Annthyroid & \textbf{0.924}$^{\color{red}*}$ & 0.877 & 0.514$^{\color{olive}\dagger}$ & 0.767 & \textit{\textbf{0.944{\color{red}*}}} & 0.747 & 0,827 & 0.764 & 0.596 & 0.617 \\
Cardio... & \textbf{0.826{\color{red}*}} & 0.812 & 0.730$^{\color{olive}\dagger}$ & 0.798 & \textit{\textbf{0.845{\color{red}*}}} & 0.776 & 0,757 & 0.626 & 0.840 & 0.540 \\
Parkinson & \textbf{0.822}{\color{red}*} & 0.777 & 0.779 & 0.745 & \textbf{0.842} & 0.838 & {\textit{0,913}} & {0.839} & 0.738 & $ot$ \\
Ionosphere & 0.927 & 0.535$^{\color{olive}\dagger}$ & \textbf{0.950} & 0.946 & 0.922$^{\color{olive}\dagger}$ & \textit{\textbf{0.970}} & {0,965} & {0.951} & 0.801 & 0.757 \\
WPBC & \textbf{0.613} & 0.591 & 0.588 & 0.574 & 0.534 & \textit{\textbf{0.635}} & 0,523 & 0.527 & 0.499 & {0.632} \\
SpamBase & 0.793{\color{red}*} & \textit{\textbf{0.848{\color{red}*}}} & {0.793{\color{red}*}} & 0.731 & \textbf{0.800{\color{red}*}} & 0.700 & 0,791 & 0.651 & 0.629 & 0.700 \\
Arrhythmia & \textbf{0.756{\color{red}*}} & 0.745 & 0.719 & 0.727 & \textit{\textbf{0.769{\color{red}*}}} & 0.734 & 0,755 & 0.740 & 0.745 & 0.724 \\ \hline
\end{tabular}
\caption{Median performance of the different one-class classification methods. DSVDD stands for DeepSVDD.}
\label{tab:unsup}
\end{table*}

We leverage hidden outliers for self-supervised one-class classification. 
By augmenting the training data with these outliers, we train a binary classifier for one-class classification, following these steps: 

\begin{enumerate}
    \item Divide the dataset $D^\textit{full}$ into $D_\textit{train}$ (80\% inliers and no outliers) and $D_\textit{test}$ (20\% inliers and all outliers). 
    \item Using $D_\textit{train}$, generate $\lvert D_\textit{train}\rvert$ outliers with an outlier generation method and add them to $D_\textit{train}$.\label{UOD:step2}
    \item Train a classifier on $D_\textit{train}$ to distinguish inliers and outliers, and evaluate its performance on $D_\textit{test}$.\label{UOD:step3} 
\end{enumerate}
We repeat the experiment seven times for each dataset with different random splits into $D_\textit{train}$ and $D_\textit{test}$. 
In Step~\ref{UOD:step2}, we use 
\textsc{Bisect}, \textsc{Hidden}, and Hyperbox (HB). 
HB is a na\"ive baseline that samples points uniformly from the minimal bounding box surrounding $D_\textit{train}$ and retains only those marked as outliers by LOF. 
In Step~\ref{UOD:step3}, we use random forest\footnote{Implementation from \texttt{caret} \cite{caret}} since it performs well in binary classification tasks \cite{rf}.

We compare this self-supervised approach with LOF and KNN adjusted for one class classification as they serve as adversaries for \textsc{Bisect}. 
We also compare it to DeepSVDD \cite{dsvdd}, a deep outlier detection method, and OCSVM, a one-class SVM with a radial kernel \cite{ocsvm}. From generative baselines, we included AnoGAN, a popular self-supervised baseline, and MO-GAAL,\footnote{Implementation from \texttt{pyod} \cite{pyod}} a recent deep-learning-based self-supervised method~\cite{AnoGAN, GAAL}. 

Table~\ref{tab:unsup} presents the test ROC AUC scores. 
We group LOF-related and KNN-related methods together and identify the best method within each group using bold font. 
If a self-supervised method significantly outperforms its adversary (p-value of the Wilcoxon signed rank test $\leq0.05$), we mark the respective AUC value with an asterisk. 
Conversely, the dagger symbol indicates when a self-supervised method performs worse than its adversary. 
Additionally, the overall best-performing method for each dataset is in italics.

We observe that Bisect-based hidden outlier methods often outperform their adversary counterparts by much. 
This suggests that we can recommend replacing conventional full-space outlier detection methods with Bisect-based self-supervised methods with respective adversaries. 
In addition, Bisect-based methods offer a runtime improvement for inference when dealing with high-dimensional datasets. 
For instance, using the SpamBase dataset, RF trained after \textsc{Bisect} with LOF as an adversary processes each test point nearly 10 times ($31$ms vs $3.4$ms) faster in average than LOF. 


\begin{table*}[ht]
\centering
\begin{tabular}{lllllllll|l}
\hline
Dataset & \textsc{Bisect} & \textsc{Hidden} & \textsc{HB} & Plain RF & cWGAN & SMOTE & DB SMOTE & ADASYN & \# out. \\ \hline
Wilt & \textbf{0.957} & 0.945 & 0.500$^{\color{olive}\dagger}$ & 0.951 & 0.953 & 0.947 & \textbf{0.956} & {0.945} & 19 \\
Pima & {0.625{\color{red}*}} & \textbf{0.700{\color{red}*}} & \textbf{0.642{\color{red}*}} & 0.589 & 0.584 & $na$ & $na$ & {$na$} & 2 \\
Stamps & \textbf{0.967{\color{red}*}} & 0.898$^{\color{olive}\dagger}$ & 0.879$^{\color{olive}\dagger}$ & 0.948 & \textbf{0.969{\color{red}*}} & $na$ & $na$ & {$na$} & 2 \\
PageBlocks & 0.982$^{\color{olive}\dagger}$ & 0.910$^{\color{olive}\dagger}$ & 0.980$^{\color{olive}\dagger}$ & \textbf{0.993} & \textbf{0.996} & 0.966$^{\color{olive}\dagger}$ & 0.980$^{\color{olive}\dagger}$ & {0.961}$^{\color{olive}\dagger}$ & 20 \\
Heart Disease & 0.814{\color{red}*} & \textbf{0.820} & \textbf{0.831} & 0.731 & 0.723 & $na$ & $na$ & {$na$} & 1 \\
Annthyroid & \textbf{0.981} & \textbf{0.990{\color{red}*}} & 0.471$^{\color{olive}\dagger}$ & 0.969 & 0.980 & 0.979 & 0.980 & {0.978} & 27 \\
Cardio... & \textbf{0.921{\color{red}*}} & 0.838 & 0.641$^{\color{olive}\dagger}$ & 0.895 & 0.916 & 0.913 & 0.904 & {\textbf{0.918}} & 7 \\
Parkinson & 0.549 & $ot$ & $ot$ & \textbf{0.590} & \textbf{0.563} & $na$ & $na$ & {$na$} & 1 \\
Ionosphere & \textbf{0.937{\color{red}*}} & 0.735 & 0.905 & 0.882 & \textbf{0.931} & $na$ & $na$ & {$na$} & 1 \\
WPBC & \textbf{0.615} & \textbf{0.640} & 0.566 & 0.602 & 0.547$^{\color{olive}\dagger}$ & $na$ & $na$ & {$na$} & 1 \\
SpamBase & 0.917{\color{red}*} & 0.868$^{\color{olive}\dagger}$ & 0.738$^{\color{olive}\dagger}$ & 0.908 & 0.903 & 0.920{\color{red}*} & \textbf{0.921}{\color{red}*} & {\textbf{0.921}}{\color{red}*} & 11 \\
Arrhythmia & \textbf{0.975{\color{red}*}} & \textbf{0.973}{\color{red}*} & 0.798{\color{red}*} & 0.689 & 0.924{\color{red}*} & $na$ & $na$ & {$na$} & 1 \\ \hline
\end{tabular}
\caption{Median performance of a random forest coupled with a generational method.}
\label{tab:balancing}
\end{table*}

\subsubsection{Supervised Outlier Detection.}

In this section, we perceive outlier detection as an imbalanced classification problem, with the train set having very few instances of the minority class (outliers). 
We use artificial hidden outliers to help balance the data, which enables us to train a binary classification model for outlier detection with the augmented set. 
The experiment has the following steps:
\begin{enumerate}
    \item Take a dataset $D$ with 2\% outliers and split it randomly into 20\% train set ($D_\textit{train}$) and 80\% test set ($D_\textit{test}$). 
    \item Add to $D_\textit{test}$ outliers from $D^\textit{full}$, excluding those in $D$, to get a more reliable outlier detection quality estimate. 
    \item Add extra outliers to $D_\textit{train}$ with outlier generation to balance outlier and inlier numbers. Do the same with oversampling, in order to have reference points.~\label{step3} 
    \item Train a classifer on $D_\textit{train}$ and evaluate it on $D_\textit{test}$.\label{step4}
\end{enumerate}
We repeat the experiment seven times for each dataset with random splits into $D_\textit{train}$ and $D_\textit{test}$. 
In Step~\ref{step4}, we employ random forest.
In Step~\ref{step3}, we use \textsc{Bisect}, \textsc{Hidden}, and Hyperbox, as before.
The oversampling methods\footnote{Implementations from \texttt{smotefamily} \cite{smotefamily}} used in Step~\ref{step3} are the neighbors-\-based approach SMOTE, the density-\-based technique ADASYN, and the clustering-\-based method DB-\-SMOTE \cite{SMOTE,Adasyn,DBSMOTE}. We also employed an outlier-based oversampling method using a cWGAN \footnote{Code by https://github.com/johaupt/GANbalanced/}, like multiple authors have proposed \cite{cWGAN,cWGAN_arxiv,cWGAN_engelmann}. 
Finally, we compare to random forest trained using only the train data. 

Table~\ref{tab:balancing} shows test ROC AUC scores. The two best methods for each dataset are in bold. 
A method significantly outperforming random forest on non-augmented data (p-value of Wilcoxon signed rank test $<0.05$) has an asterisk, while an inverse case is marked with the dagger symbol. 
The last column contains the outlier count in $D_\textit{train}$.

The \textsc{Bisect} method is a clear winner overall, while conventional oversampling techniques are not applicable for datasets with sparse outliers as they have less than 3 minority class objects (marked as $na$). The only other method thought for scarce examples of the minority class, cWGAN, failed to significantly improve the classifier as much as our approach.
Hence, we also recommend the \textsc{Bisect} method for oversampling in the case with few recorded outliers.


\subsubsection{Limitations and Future Work.}\label{competitors}

To cover recent and popular methods, our experiments include various competitors for supervised outlier detection and one-class classification. 
However, we faced difficulties including certain methods due to either outdated or unavailable implementations~\cite{ocrf, active,cWGAN}, or implementations designed solely to replicate outcomes on specific benchmark datasets~\cite{Oversamplingwithoutliers, Dlamini2021}. 

In addition, there are other approaches, largely or completely unexplored for these tasks but adaptable like our use of \textsc{Bisect}. 
For example, beyond the Hyperbox and \textsc{Hidden} methods we have included in our comparison, 17 outlier generation methods reviewed in~\cite{Outsurvey} fall into this category.
The experimental design space could also extend to other conventional outlier detection methods from \cite{Campos2016}~--- beyond LOF, KNN, and OCSVM which we have covered. 
These could serve as further adversaries of \textsc{Bisect} or \textsc{Hidden}, or as competitors in one-class classification. 
However, such a broad comparison exceeds the scope of one conference publication, and we see it as future work. 



\section{Conclusions}
Generating outliers is useful. 
Nevertheless, most outlier generation methods disregard the ``multiple-views'' property of outliers in high-dimensional spaces. 
Synthetic hidden outliers are the sole exception. 
However, the existing method for generating hidden outliers is inefficient and sensitive to hyperparameter selection, as we have shown. Furthermore, the utility of hidden outliers remains to be shown.

In this paper, we have established that synthetic hidden outliers exist 
under versatile conditions 
and, based on this theory, propose a way to search for respective candidates efficiently, which we call \textsc{Bisect}.
\textsc{Bisect} is notably faster than the current alternative. 

Next, we developed a methodology that makes use of outliers generated with \textsc{Bisect} to enhance outlier detection in the context of one-class classification and highly unbalanced supervised outlier detection tasks. 
In both scenarios, the use of \textsc{Bisect} yielded significant improvements over conventional methods, surpassing widely adopted alternatives tailored for these tasks. 
These results confirm the potential of hidden outliers to advance various outlier detection tasks.

\section{Aknowledgments}
This work was supported by the Ministry of Science, Research and the Arts Baden Württemberg, project Algorithm Engineering for the Scalability Challenge (AESC).

\bibliography{bib}{}

\appendix

\section{Theoretical Appendix}
This appendix includes all the proofs for the theoretical results presented in Section 3, as well as other supplementary results needed to prove the said results. We include all the proofs in the order of inclusion in the text. For completeness, the statements are repeated before the proof.

Additionally, we are going to introduce a couple of notations needed for the proofs. First, let us consider everything from Section 3. Now, let $C(a,b)$ be the image of the convex combination of two points of $X$, $a$, and $b$. Unless is stated otherwise, $\alpha(t) = ty + (1-t)x$. Additionally, $C(a,b)_-$ is the image of the convex combination with $t\in(0,1]$ and $C(a,b)^+$ with $t \in [0,1)$. The set $C(a,b)_-^+$ is the convex combination with $t\in (0,1)$.

Whenever we write $B_X(z,\varepsilon)$ we refer to an open ball of $X$ centered in $z$ with radius $\varepsilon>0$. $B_A(z,\varepsilon)$, with $A\subset X$, stands for the open ball in the induced topology, i.e. $B_A(z,\varepsilon) = B_X(z,\varepsilon) \cap A$.

\begin{proposition}("Hidden outlier existence"):
    Let $x$ and $y$ be points in the previously defined metric space such that $x \in R(\mathcal{M})$ and $y \notin R(\mathcal{M})$. Assume that there exists a point $z$ in the convex combination of $x$ and $y$ such as $z \in \partial R(\mathcal{M}) \Rightarrow z \notin \partial R(\ensemble{M})$. 
    Then there exists 
    $z'$ in the convex combination of $x$ and $y$ such as $z' \in H_1(\mathcal{M}) \cup H_2(\mathcal{M})$.
\end{proposition}
\begin{proof}
    Assume $z \in \partial R(\mathcal{M})$, otherwise $z \in R(\mathcal{M})\setminus  R(\ensemble{M})$ trivially. Also, since $z\notin\partial R(\ensemble{M})$ by hypothesis, $z$ has to be in the interior, $ R(\ensemble{M})^\circ.$
    Now, since $z \in \partial R(\mathcal{M})$, by definition of boundary point, $\forall \varepsilon > 0,~B_X(z,\varepsilon)\cap(X\setminus R(\mathcal{M})) \neq \emptyset.$ Recall that, by the notion of induced topology, \[C(x,y)\subset X \Longrightarrow B_{C(x,y)}(z,\varepsilon) = B_X(z, \varepsilon) \cap C(x,y).\]
    Then, since $z\in \partial R(\mathcal{M})\cap C(x,y)$, we have
    \begin{equation}\label{cond1}
        \forall \varepsilon > 0,~ B_{C(x,y)}(z,\varepsilon) \cap X \setminus  R(\mathcal{M}) \neq \emptyset.
    \end{equation}
    Since,
    \[
        z \in R(\ensemble{M})^\circ \Longrightarrow \exists\varepsilon'>0,~ B_{C(x,y)}(z,\varepsilon') \subset R(\ensemble{M}),
    \]
    we have that \begin{equation}\label{cond2}
        B_{C(x,y)}(z,\varepsilon') = B_{C(x,y)}(z,\varepsilon')\cap R(\ensemble{M}).
    \end{equation}
    By (\ref{cond1}), \[B_{C(x,y)}(z,\varepsilon')\cap X\setminus R(\mathcal{M})\neq\emptyset.\] By (\ref{cond2}),  \[B(z,\varepsilon')\cap X\setminus R(\mathcal{M}) = B(z,\varepsilon')\cap X\setminus R(\mathcal{M})\cap R(\ensemble{M}) \neq \emptyset.\]
    By Zornn's Lemma, we can find $z'\in X$ such that\[
    z \in B_{C(x,y)}(z,\varepsilon')\cap R(\ensemble{M})\setminus  R(\mathcal{M})=C(x,y)\cap H_2.
    \]
    Changing $z\in \partial R(\mathcal{M})$ to $z\in\partial R(\ensemble{M})$ leads to the inclusion to $H_1$.
\end{proof}

Before proving Theorem 1 we need to introduce a Lemma that is going to help us through the proof. It will help us assess when one should expect a boundary point on the convex combination of two points, and under what conditions.
\begin{lemma}\label{lemma1}
    Consider $X$ as our metric space as before, and $A \in X$ a subspace. Then,
    \begin{equation}\label{lemma1:r1}
        \begin{cases}
            x \in A,\\
            y \notin A.
        \end{cases} \Longrightarrow \exists z \in \partial A\cap C(x,y).
    \end{equation}
    \begin{equation}\label{lemma1:r2}
        \begin{cases}
            x \in A, \\
            y \in A,\\
            \nexists z \in \partial A \cap C(x,y).
        \end{cases} \Longrightarrow C(x,y) \subset A.
    \end{equation}
\end{lemma}
\begin{proof}
    First, let us prove statement \ref{lemma1:r1}. Let $\{t_i\}_{i\in  \mathbb{N}}$ be a monotonone sequence in $[0,1]$ such that $\alpha(t_i) \in A$ for all points in the sequence. As $y \notin A,$ very clearly such sequence has to be bounded. As $\{t_i\}$ is bounded, then, it has a supremum $T$. By the monotonic convergence theorem, \[
    t_i \longrightarrow T.
    \]
    Assume that $\alpha(T) \notin \partial A$. Then, you can fit an open ball with radius $\varepsilon>0$ in $A\cap C(x,y)$, and get a point $z' = \alpha(T')$ in the ball such that $T' = T + \varepsilon > T$. As $T$ is the supremum, this cannot happen. Therefore, $\alpha(T) = z\in \partial A$.

    Let us now prove statement \ref{lemma1:r2}. Assume that it exists a $t$ in $[0,1]$ such that $\alpha(t) \notin A$. Then, by the statement \ref{lemma1:r1} of this very same lemma, $\exists z \in \partial A$, which is untrue by hypothesis. 
\end{proof}

Now, we can finally tackle the proof for Theorem 1. 
\begin{theorem}("Convergence into a hidden outlier")
Let $f$ be defined as before. 
Assume that at most exist, and are unique, other $z \text{ and } z'$ in the convex combination of $x\in R(\mathcal{M})$ and $y \notin R(\mathcal{M})$ such as $z \in \partial R(\mathcal{M})$, $z' \in \partial R(\ensemble{M})$, $z \neq z'$, and both verify the last condition of proposition \ref{convexexistence}.
Then the bisection method will converge to a root of $f$.
\end{theorem}
\begin{proof}
    The way of proving this statement will be to consider all inclusion possibilities and check when we can converge to the function $f$ zeroes using the bisection method. Recall that the bisection method always converges as long as there is a sequence of nested intervals $[a_1,b_1]\supset\cdots\supset[a_n,b_n]\supset\cdots$ such that the root is always in the sequence, and are such that $sign(f(a_i)) \neq sign(f(b_i))$. We will first break down all possible inclusions from the different acceptance regions. After that, we will study when the function $f$ converges into a hidden outlier by looking at all its possible values.
    
    First, let us assume that $x \in R(\ensemble{M})$. The case $x \notin R(\ensemble{M})$ will be proven afterward.
    \paragraph{Part I $(x \in R(\ensemble{M})):$}
     Consider that $z\in R(\mathcal{M})$ since $R(\mathcal{M})$ is closed and $x\in R(\mathcal{M})$ by hypothesis. Additionally, $\nexists z'' \in \partial R(\mathcal{M})$ in $C(x,z)$ as $z$ is the only boundary point in $C(x,z) \subset C(x,y)$ by hypothesis. Then, we have that $C(x,z) \subset R(\mathcal{M})$ by Lemma \ref{lemma1}. Now, we have two possibilities:

    \begin{enumerate}
        \item $z \in R(\ensemble{M}):$ As $z$ hast to be an interior point by hypothesis, $C(x,z)$ is completely contained in $R(\mathcal{M}) \cap R(\ensemble{M}).$ Assuming that: \begin{itemize}
            \item[a.] $y\in R(\ensemble{M})$: Assume that there are no $z' \in \partial R(\ensemble{M})$ in $C(z,y)$. Therefore by Lemma 1, we have that $\forall c \in C(x,y), C(c,y) \subset R(\ensemble{M})$, as long as $c \neq z$.
            \item[b.] $y \notin R(\ensemble{M})$: As $z \in R(\ensemble{M})$ and $y \notin R(\ensemble{M})$,  by Lemma \ref{lemma1} we can find $z' \in \partial R(\ensemble{M})$. Additionally, by hypothesis, $z'\notin \partial R(\mathcal{M})$, and both $z$ and $z'$ are the unique boundary points in $C(x,y)$. That also means that they are the only boundary points in $C(z,z')$. Now, as there are no boundary points in $C(z,z')$, we have that $\forall c \in C(z,z'), c\neq z$ are not points of $R(\mathcal{M})$ thanks to statement \ref{lemma1:r2} from Lemma \ref{lemma1}. Lastly, is instant by the same argument that $C(c,y) \subset X\setminus(R(\mathcal{M} \cup R(\ensemble{M})$.
        \end{itemize}
        Then, in case 1.a, $C(x,y) = C(x,z) \cup C(z,y)$ and $f(C(x,z)) = -1,~ f(C(z,y)) = 0.$ 
        For case 1.b, $C(x,y) = C(x,z) \cup C(z,z') \cup C(z',y)$ and $f(C(x,z)) = - 1,~ f(C(z,z')_{-}) = 0,~ f(C(z',y)_{-}) = 1.$
        \item $z \notin R(\ensemble{M})$: As $z \notin R(\ensemble{M})$ and $x \in R(\ensemble{M})$, by Lemma \ref{lemma1} and hypothesis, there exists, and is unique in $C(x,y)$, a point $z'\in \partial R(\ensemble{M})$ in $C(x,z)$.
        \begin{itemize}
            \item[a.] $y \in R(\ensemble{M})$: This cannot occur since it leads to a contradiction by using Lemma 1 and obtaining a different $z''$ in the boundary of $R(\ensemble{M})$.
            \item[b.] $y \notin R(\ensemble{M})$: By the statement \ref{lemma1:r2} from Lemma \ref{lemma1} we have that $C(z,y)\subset X\setminus R(\ensemble{M})$. Additionally, using the same argument as in 1.b, $C(c,y)\subset X\setminus R(\mathcal{M})$ for all $c\in C(z,y)$ such that $c\neq z$. By utilizing statement \ref{lemma1:r1} again with the same argument as in 1.b, but changing $z\in \partial R(\mathcal{M})$ to $z' \in \partial R(\ensemble{M})$, get that $C(x,z')\subset R(\mathcal{M})\cap R(\ensemble{M})$. 
        \end{itemize}
        Thus, we have that for case 2.b $f(C(x,z')) = -1$, $f(C(z',z)_-) = 0$, $f(C(z,y)_-) = 1$.
    \end{enumerate}
    \paragraph{Part II $(x \notin R(\ensemble{M})):$} By Lemma \ref{lemma1} and hypothesis, $C(x,z) \subset R(\mathcal{M})$. Consider again: 
    \begin{enumerate}
        \item $z\in R(\ensemble{M})$: By Lemma1, there exists $z' \in \partial R(\ensemble{M})$ in the convex combination between $x$ and $z$. As there are no more boundary points left, we have that $C(x,z')^+ \subset R(\mathcal{M})\setminus R(\mathcal{M})$ and $C(z',z) \subset R(\mathcal{M})\cap R(\ensemble{M})$. Again, let us divide the hypothesis space into two cases:
        \begin{itemize}
            \item[a.] $y\in R(\ensemble{M})$: By pretty similar arguments as in I.1.b and before, $C(z,y)_- \subset R(\ensemble{M})\setminus R(\mathcal{M})$.
            \item[b.] $y\in R(\ensemble{M})$: As in 2.b, this case is impossible, otherwise it will lead to a contradiction of the hypothesis, just like in I.2.b.
        \end{itemize}
        In this case, we will have for 1.a that $f(C(x,z')^+=0,~f(C(z',z) = -1,~ f(C(z,y)) = 0$.
        \item $z \notin R(\ensemble{M})$ \begin{itemize}
            \item[a.] $y \in R(\ensemble{M}):$ There has to exist $z' \in \partial R(\ensemble{M})$ in $C(z,y)$ by the statement \ref{lemma1:r1} of Lemma \ref{lemma1}. Then, as there are no more boundary points, we can use statement \ref{lemma1:r2} to obtain that $C(x,z) \subset R(\mathcal{M})\setminus R(\ensemble{M})$,  $C(z,z')_-^+ \subset X\setminus(R(\mathcal{M})\cup R(\ensemble{M}))$, and that $C(z',y) \subset R(\ensemble{M})\setminus R(\mathcal{M})$.
            \item[b.] $y \notin R(\ensemble{M})$: Lastly, let us assume that $\nexists z' \in C(x,y)$ such that $z' \in \partial R(\ensemble{M})$, as we did in I.1.a. Then, $C(z,y) \subset X\setminus(R(\ensemble{M}) \cup R(\mathcal{M}))$. 
        \end{itemize}
        Therefore, for 2.a: $f(C(x,z)) = 0,~ f(C(z,z')_-) = 1,~ f(C(z',y)) = 0$. And, finally, for 2.b: $f(C(x,z))=0,~f(C(z,y)) = 1$.
    \end{enumerate}
    That way, in cases I.1.b, I.2.b, II.1.a, and II.2.a, it is fairly obvious how to construct a sequence such that we can encapsulate all roots. For cases I.1.b and I.2.b it suffices with selecting the hole interval. For cases II.1.b and II.2.a it suffices to restrict the function $f$ to the convex combination of a point $c\in C(z,z')$ and $x$. For cases I.1.a and II.2.b, if one assumes that there exists another boundary point $z' \partial R(\ensemble{M})$ also in $C(z,y)$ it can be proven similarly as for cases I.1.a and II.2.b. In both cases $z'$ is the only point from $R(\ensemble{M})$ in the convex combination, otherwise by Lemma \ref{lemma1} we could get another boundary point and get to a contradiction. Then it suffices to take again a point $c \in C(z,z')^+$ and restrict the $f$ to the convex combination of $x$ and $c$, as before. Utilizing this idea of segmenting the intervals by a closer outlier of $\mathcal{M}$ could be extended to tackle the case where there are more than one $z$ and $z'$.

    This proves the convergence of the bisection method. However, consider that $length([a_n,b_n])\longrightarrow 0$ with the bisection method. Consider as well that for every $h$ root of $f$ there exists a neighborhood $N(h)$ of length greater than 0. Then, $N(h)$ is not the supremum of the sequence $[a_n,b_n]$, and therefore, the bisection method will converge in finite time into any point $h'$ from $N(h)$. I.e., it will find hidden outliers in finite time. 
\end{proof}

\begin{proposition}
    Let $n_\text{subs}$ be the number of selected subspaces for the ensemble. 
    The worst-case complexity of \textsc{Bisect} is  $\mathcal{O}\left(\log(\frac{L}{n_\text{cuts}}) \cdot (n_\text{subs}+n_\text{cuts}) \cdot \aleph\right)$ where $\aleph$ is the inference complexity of the adversary $\mathcal{M}(\cdot)$. 
\end{proposition}
\begin{proof}
Consider that we obtained a $L$ sufficiently large as stated in the steps for \textsc{Bisect}. Then, complexity can be trivially bounded by $\mathcal{O}\left(n_\text{iter}\cdot (n_\text{subs}+n_\text{cuts})\cdot\aleph\right)$, with $n_\text{iter}$ being the number of iterations for the bisection method. As the bisection method is also bounded by $\mathcal{O}\left(\log(L')\right)$ being $L'$ the length of the interval (in our case $L' = \frac{L}{n_\text{cuts}}$), we have that:\[
\mathcal{O}\left(\log(L')\cdot(n_\text{subs}+n_\text{cuts})\cdot\aleph\right) \leq \mathcal{O}\left(n_\text{iter}\cdot (n_\text{subs}+n_\text{cuts})\cdot\aleph\right).
\]

Let us derive the bounding for the bisection method for completion. Assume that we want to converge to the right side of the interval w.l.g, as the total length traveled inside the interval will be equivalent. Consider, \[
\sum_{i=1}^{n_I} \frac{x_i - x_{i-1}}{2} = L' - Err,
\]
where $n_I$ is a finite number of iterations, and $Err$ the error of the algorithm. Then, as we want to convert to the right side, $L'$:\[
\sum_{i=1}^{n_I} \frac{L'}{2^i} = L'-Err.
\]
We can rewrite our series as a finite geometric series by $\frac{1}{2}\sum_{i=0}^{n_I}\frac{L'}{2^i} = \sum_{i=1}^{n_I} \frac{L'}{2^i}$. Then, by considering its sum:
\begin{equation}\label{sum}
    \frac{1}{2} \left(\frac{1-\frac{1}{2^{n_I + 1}}}{1-\frac{1}{2}}\right) = L' - Err.
\end{equation}
By (\ref{sum}), and after doing some algebra, \[
Err = \frac{L'}{2^{n_I + 1}},
\]
which leads to the desired bounding. Additionally, one could obtain the number of iterations to converge with an error for a desired interval length with (\ref{sum}). 
\end{proof}

\section{Experimental Appendix}

In this Appendix, we include the experimental results supporting the selection of $\varepsilon$. In Table \ref{fig:time_eps} we collected the time for generating 100 hidden outliers on each task-specific training set. We selected the $\varepsilon$ with the smaller maximum execution time between a small, medium, and large value of epsilon ($0.1,~0.5$ and $0.75$, respectively) in this case it was $\varepsilon = 0.1$. Additionally, we also performed each experiment with the corresponding fastest epsilon between the tested. Results were gathered in Table \ref{tab:opt_SOD} for Supervised Outlier Detection and in Table \ref{tab:opt_OCC} for One-class Classification. We bolded those results with higher median AUC in each row. We marked with an asterisk those that were significantly different by the Wilcoxon-signed rank test, as in Section 4. As we can see from both tables, there are no significant differences between both values. Sometimes $\varepsilon_\text{opt}$ degrades the performance of the classifier and sometimes increases it compared to the smaller value of $\varepsilon$.
\begin{table}
\centering
\begin{tabular}{lll}
\hline
Dataset & $\varepsilon = 0.1$ & $\varepsilon_\text{opt}$ \\ \hline
Stamps & \textbf{0.942}* & 0.853 \\
Annthyroid & \textbf{0.944}* & 0.793 \\
Cardio... & 0.812 & \textbf{0.815} \\
Parkinson & 0.777 & \textbf{0.867}* \\
Ionosphere & 0.535 & \textbf{0.969}* \\
WPBC & 0.591 & \textbf{0.644} \\
SpamBase & 0.848 & \textbf{0.864} \\ \hline
\end{tabular}
\caption{Median AUC for \textsc{Hidden} with $\varepsilon = 0.1$ and optimal $\varepsilon$ in time, in Supervised Outlier Detection.}
\label{tab:opt_SOD}
\end{table}
\begin{table}
\centering
\begin{tabular}{lll}
\hline
Dataset & $\varepsilon = 0.1$ & $\varepsilon_\text{opt}$ \\ \hline
Stamps & \textbf{0.898}* & 0.699 \\
Annthyroid & \textbf{0.990}* & 0.944 \\
Cardio... & \textbf{0.838} & 0.815 \\
Parkinson & $ot$ & $ot$ \\
Ionosphere & 0.735 & \textbf{0.97}* \\
WPBC & 0.566 & \textbf{0.644} \\
SpamBase & 0.738 & \textbf{0.864}* \\ \hline
\end{tabular}
\caption{Median AUC for \textsc{Hidden} with $\varepsilon = 0.1$ and optimal $\varepsilon$ in time, in One-class Classification.}
\label{tab:opt_OCC}
\end{table}
\begin{table*}
\centering
\begin{tabular}{llll|lll}
\hline
\multicolumn{1}{c}{\multirow{2}{*}{Datasets}} & \multicolumn{3}{c|}{One-class classification} & \multicolumn{3}{c}{Supervised outlier detection} \\
\multicolumn{1}{c}{} & $.1$ & $.5$ & $.75$ & $.1$ & $.5$ & $.75$ \\ \hline
Wilt & \textbf{4.28} & 10.80 & 26.89 & \textbf{3.978} & 10.646 & 24.776 \\
Pima & \textbf{0.852} & 1.512 & 4.214 & \textbf{0.596} & 1.039 & 2.539 \\
Stamps & 8.285 & \textbf{5.428} & 12.719 & 13.18 & \textbf{5.641} & 10.018 \\
PageBlocks & \multicolumn{1}{c}{-} & \multicolumn{1}{c}{-} & \multicolumn{1}{c|}{-} & \textbf{4.731} & 131.461 & 216.037 \\
Heart Disease & \textbf{1.151} & 1.195 & 1.529 & \textbf{1.1} & \textbf{1.1} & 1.267 \\
Annthyroid & 5.698 & \textbf{4.217} & 4.577 & 4.837 & 4.226 & \textbf{2.267} \\
Cardiotocography & 20.661 & 3.847 & \textbf{3.705} & 15.638 & 4.588 & \textbf{4.559} \\
Parkinson & 76.473 & \textbf{20.685} & 27.01 & \multicolumn{1}{c}{-} & \multicolumn{1}{c}{-} & \multicolumn{1}{c}{-} \\
Ionosphere & 38.38 & 8.524 & \textbf{7.793} & 53.485 & 9.82 & \textbf{8.784} \\
WPBC & 95.828 & \textbf{7.264} & 25.046 & 90.687 & \textbf{7.609} & 21.963 \\
SpamBase & 14.444 & 12.406 & \textbf{11.004} & 12.606 & 13.127 & \textbf{11.878} \\
Arrhytmia & \textbf{10.277} & 23.696 & 15.58 & \textbf{11.819} & 38.832 & 25.815 \\ \hline
\end{tabular}
\caption{Time to generate 100 hidden outliers for each dataset using the training set for each use case.}
\label{fig:time_eps}
\end{table*}

\end{document}